\documentclass[lettersize,journal]{IEEEtran}
\usepackage{amsmath,amsfonts}
\usepackage[ruled,linesnumbered]{algorithm2e}
\usepackage{array}
\usepackage[caption=false,font=normalsize,labelfont=sf,textfont=sf]{subfig}
\usepackage{textcomp}
\usepackage{stfloats}
\usepackage{url}
\usepackage{verbatim}
\usepackage{graphicx}
\usepackage{cite}
\usepackage{hyperref}
\usepackage{multirow}
\usepackage{graphicx}
\usepackage{caption}
\usepackage{float}
\usepackage{amsmath}
\usepackage{amsthm}
\usepackage{xcolor}
\newtheorem{myDef}{Definition}
\newtheorem{theorem}{Theorem}
\newtheorem{lemma}{Lemma}
\usepackage[para,online,flushleft]{threeparttable}
\usepackage{array}

\begin{document}

\title{DiffLoad: Uncertainty Quantification in Electrical \\Load Forecasting with the Diffusion Model}

\author{Zhixian Wang, Qingsong Wen, Chaoli Zhang, Liang Sun, and Yi Wang
\thanks{The work was supported in part by the National Key R\&D Program of China (2022YFB2403300), in part by the Applied Basic Research Foundation (2024A1515011266), and in part by the Young Elite Scientists Sponsorship Program by CAST. (\textit{Corresponding author: Yi Wang})}
\thanks{Zhixian Wang and Yi Wang  are with the Department of Electrical and Electronic Engineering, The University of Hong Kong, Hong Kong SAR, China (e-mail: zxwang@eee.hku.hk, yiwang@eee.hku.hk).}
\thanks{Qingsong Wen and Liang Sun are with the DAMO Academy, Alibaba Group (U.S.) Inc., Bellevue, WA 98004, USA (e-mail: qingsong.wen@alibaba-inc.com, liang.sun@alibaba-inc.com).}
\thanks{Chaoli Zhang is with the DAMO Academy, Alibaba Group, Hangzhou, China (e-mail: chaoli.zcl@alibaba-inc.com).}
}

\markboth{Submitted to IEEE Transactions on Power Systems}%
{Shell \MakeLowercase{\textit{et al.}}: A Sample Article Using IEEEtran.cls for IEEE Journals}

\maketitle

\begin{abstract}
Electrical load forecasting plays a crucial role in decision-making for power systems. The integration of renewable energy sources and the occurrence of external events, such as the COVID-19 pandemic, have rapidly increased uncertainties in load forecasting. The uncertainties in load forecasting can be divided into two types: epistemic uncertainty and aleatoric uncertainty. Modeling these types of uncertainties can help decision-makers better understand where and to what extent the uncertainty is, thereby enhancing their confidence in the following decision-making. This paper proposes a diffusion-based Seq2seq structure to estimate epistemic uncertainty and employs the robust additive Cauchy distribution to estimate aleatoric uncertainty. Our method not only ensures the accuracy of load forecasting but also demonstrates the ability to separate and model the two types of uncertainties for different levels of loads. The relevant code can be found at \url{https://github.com/hkuedl/DiffLoad-Uncertainty-Quantification-Load-Forecasting}.
\end{abstract}


\section{Introduction}

\subsection{Background and Motivation}\label{motivation}
Electrical load forecasting is vital in helping power system decision-making. In recent years, Neural Network (NN)-based load forecasting methods have been widely applied~\cite{hammad2020methods,zhu2023energy}. However, the uncertainties in NN-based load forecasting may reduce decision-makers' trust in the forecasts. The uncertainties can be divided into two parts: epistemic uncertainty and aleatoric uncertainty. For load forecasting, quantifying these two types of uncertainties corresponds to different scenarios and has different applications. Quantifying aleatoric uncertainty, what most probabilistic load forecasting does, can leave a margin for ordinary power grid dispatching. Epistemic uncertainty, on the other hand, quantifies whether the models understand the data well. When the power system suffers from external shocks such as the COVID-19 epidemic (shown in Fig. \ref{COV}) and extreme weather, the load data may shift and forecasting errors may increase rapidly. Quantifying this uncertainty can help downstream task decision-makers understand the relevant risks to reduce economic losses. Therefore, to increase decision-makers' confidence in forecasts, the model should be able to model/quantify these uncertainties.

\begin{figure}[t]
\centering
\includegraphics[width=0.5\textwidth]{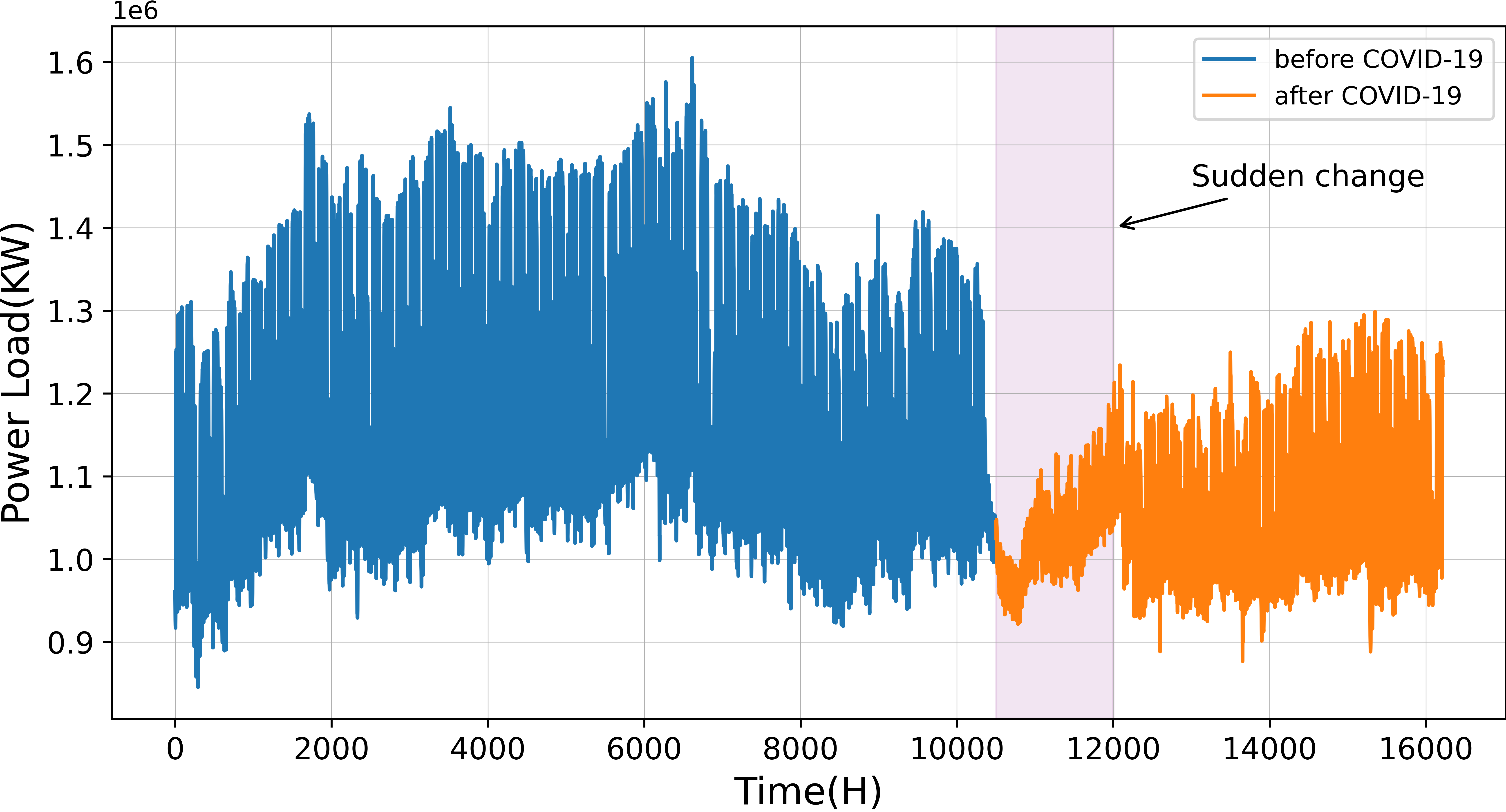}
\centering
\caption{Visualization of sudden change effects caused by external events in electricity load}
\label{COV}
\end{figure}

\subsection{Literautre Review}
Extensive work has been done on load forecasting, which can be roughly divided into two categories. The first is statistics-based methods, which explicitly construct the relationship between input data (such as historical load) and the load to be forecasted. \cite{vaghefi2014modeling} adopted temperature as an exogenous variable and built autoregressive moving average models with exogenous inputs (ARMAX) to forecast load. The second is machine learning-based methods, which learn latent patterns from known data and apply them to unknown data. Traditional machine learning methods such as linear regression and tree models are widely applied. \cite{lindberg2019modelling} used time, external variables (such as temperature), and day type (holiday or not) to construct a linear regression model. \cite{wang2021short} decomposed the load series into trend series and multiple fluctuation subsequences and then constructed several linear regression models and Xgboost regression models to forecast each series separately. Due to the relatively simple model structure, traditional machine learning methods can usually explain how input variables affect the forecasts well. However, the simple structure often can not cope with strong nonlinearity in reality. To this end, researchers have gradually turned their attention to NN in recent years~\cite{almalaq2017review,zhou2023robust,grabner2023global}. The Long Short Term Memory (LSTM) recurrent NN was applied to model the power load of individual electric customers~\cite{kong2017short}, demonstrating the superiority of NN in load forecasting tasks. In addition to using existing NN, researchers also considered modifying NN for load forecasting. \cite{li2021attention} proposed an NN architecture with an attention mechanism for developing RNN-based building energy forecasting and investigated the effectiveness of this attention mechanism in improving interpretability. \cite{jiao2021adaptive} designed a method to handle external weather and calendar variables jointly and then combined them with LSTM.

Although NN is widely used in load forecasting, various issues like data noise and unknown external influences (such as COVID-19) will make it difficult to model the load data well, therefore bringing huge uncertainties. This makes researchers shift their focus from point forecasting to probabilistic forecasting, which essentially models the uncertainties in forecasting~\cite{journel1994modeling}. \cite{wang2019probabilistic} proposed to use Pinball Loss to guide the LSTM so that it can output the quantile of the data. In \cite{yang2019deep}, a novel deep ensemble learning-based probabilistic load forecasting framework was proposed to quantify the load uncertainties of individual customers. \cite{xu2019probabilistic} considered peak areas with often more significant uncertainty and supposed that the load data consists of probabilistic normal load and the probabilistic peak abnormal differential load. Apart from the methods mentioned above, a recent trend is to employ the diffusion mechanism to model the uncertainty. TimeGrad \cite{rasul2021autoregressive} and \cite{shen2023non} used NN to extract information from time series data and assist in constructing Markov Chains, ultimately obtaining uncertainty estimations that do not require hypothetical distributions.

Although these practices can provide probabilistic forecasts to capture uncertainties, they did not clearly define what uncertainty they were modeling and therefore can not provide further insights into the forecasting process. \cite{kendall2017uncertainties} claimed that the uncertainty of NN forecasting models can be categorized into epistemic uncertainty caused by the forecasting model and aleatoric uncertainty caused by the data itself. The existing common methods, such as deepAR~\cite{salinas2020deepar}, are probabilistic modeling of forecasting errors, considering the aleatoric uncertainty. In recent years, with the gradual increase in the penetration rate of renewable energy generation, the randomness of meteorological factors has affected renewable energy generation, further affecting the uncertainty of load forecasting~\cite{sun2019using}. The significance of this type of probabilistic load forecasting is mainly to leave a margin for ordinary power grid dispatching~\cite{nosair2017economic}. Epistemic uncertainties, on the other hand, arise from
incomplete information in the training set ~\cite{rawal2023load}. When faced with deviations in the distribution of training and test data, the existing method usually can not provide a reasonable probabilistic forecasting interval, which may lead to a huge economic loss, because it ignores the epistemic uncertainty. In recent years, the impact of extreme external events such as COVID-19 on load forecasting and downstream power grid dispatching tasks has received increasing attention~\cite{zhao2023gaussian}, indicating that we not only need to consider the impact of aleatoric uncertainty caused by data noise on normal situations but also the impact of epistemic uncertainty caused by external events. For now, some methods have been proposed to deal with different types of uncertainties, such as Bayesian methods, ensemble methods, and dropout. Bayesian methods assumed that different types of uncertainty followed Gaussian distributions. \cite{yang2019bayesian} proposed a novel probabilistic day-ahead net load forecasting method to capture epistemic and aleatoric uncertainty using Bayesian deep learning. Similarly, \cite{sun2019using} also applied Bayesian NN to capture two types of uncertainties, and the results were used in subsequent pooling clustering, ultimately improving forecasting accuracy. In addition to the Bayesian NN, dropout~\cite{gal2016dropout}, a common technique for training NN, has also been proved to be an approximation of the Bayesian network, thus providing two types of uncertainty estimates. Because epistemic uncertainty represents model uncertainty during the training process, ensemble methods~\cite{lakshminarayanan2017simple} have become one of the methods for estimating model uncertainty. However, these methods all had their drawbacks. For the ensemble approach, the time and computational costs were very expensive due to the need to train multiple models. Similarly, the Bayesian method treated each NN parameter as a random variable, making the training cost extremely expensive. Meanwhile, these two methods typically relied on Gaussian distributions, which limited the model's expressive power and was easily affected by data noise \cite{huber2011robust}. As for dropout, the advantage of this method was that it did not require assumptions about the distribution of uncertainty, and compared to ensemble-based and Bayesian methods, it reduced the computational time. However, it has been proven that its forecasting performance was unstable due to inconsistencies in the training and testing processes~\cite{li2019understanding}.  
\subsection{Contributions}
Given that existing methods require significant computational resources or are susceptible to issues like data noise because of the Gaussian distribution, we are motivated to develop a new uncertainty quantification framework. The framework can estimate and separate the two kinds of uncertainties while reducing distribution assumptions and does not significantly increase the computational burden. To estimate the aleatoric uncertainty, we propose to apply a heavy-tailed emission head to wrap up the forecasting model, reducing the bad effect caused by data noise. As for the epistemic uncertainty, we propose a diffusion-based framework to concentrate the uncertainty of the model on the hidden state, which only increases the computational burden that can be borne. Combining two types of uncertainties, our forecasting model will provide high-quality load forecasts. This paper makes the following contributions:
\begin{itemize}
\item Provide a new epistemic uncertainty quantification framework in electrical load forecasting: Based on sequence-to-sequence (Seq2Seq) and diffusion structure, we propose a new uncertainty quantification method for NN forecasting models. Unlike previous methods that set model parameters to random variables like Bayesian NN and dropout, we utilize a diffusion-based encoder to concentrate uncertainty in the hidden layer of the NN before inputting it into the decoder. In this way, our method provides estimations of model uncertainty while bringing an affordable additional computational burden.

\item Propose a robust emission head to capture the aleatoric uncertainties: A likelihood model based on the additive Cauchy distribution is proposed to estimate the uncertainty of the data. Compared with the traditional Gaussian likelihood, the Cauchy distribution is more robust to the outlier and extreme values of the power load data.

\item Conduct extensive experiments including load data at different levels: We compared our methods with the widely used uncertainty quantification method for NN at different levels of load data. The experiment result shows that our method outperforms traditional methods at both system-level and building-level loads without adding a significant computational burden. The experiment's code can be found in \url{https://github.com/hkuedl/DiffLoad-Uncertainty-Quantification-Load-Forecasting}.
\end{itemize}

\subsection{Paper Organization}
The rest of the paper is structured as follows. Section \ref{Proposed Method} elaborates our proposed method, including how to separate two types of uncertainty using diffusion structure and robust Cauchy distribution. Section \ref{Example analysis} reports experimental results and analysis. Section \ref{Conclusion} gives conclusions and directions for further research.

\section{Proposed Method}\label{Proposed Method}
\subsection{Framework}

\begin{figure}[t]
\centering
\includegraphics[width=0.5\textwidth]{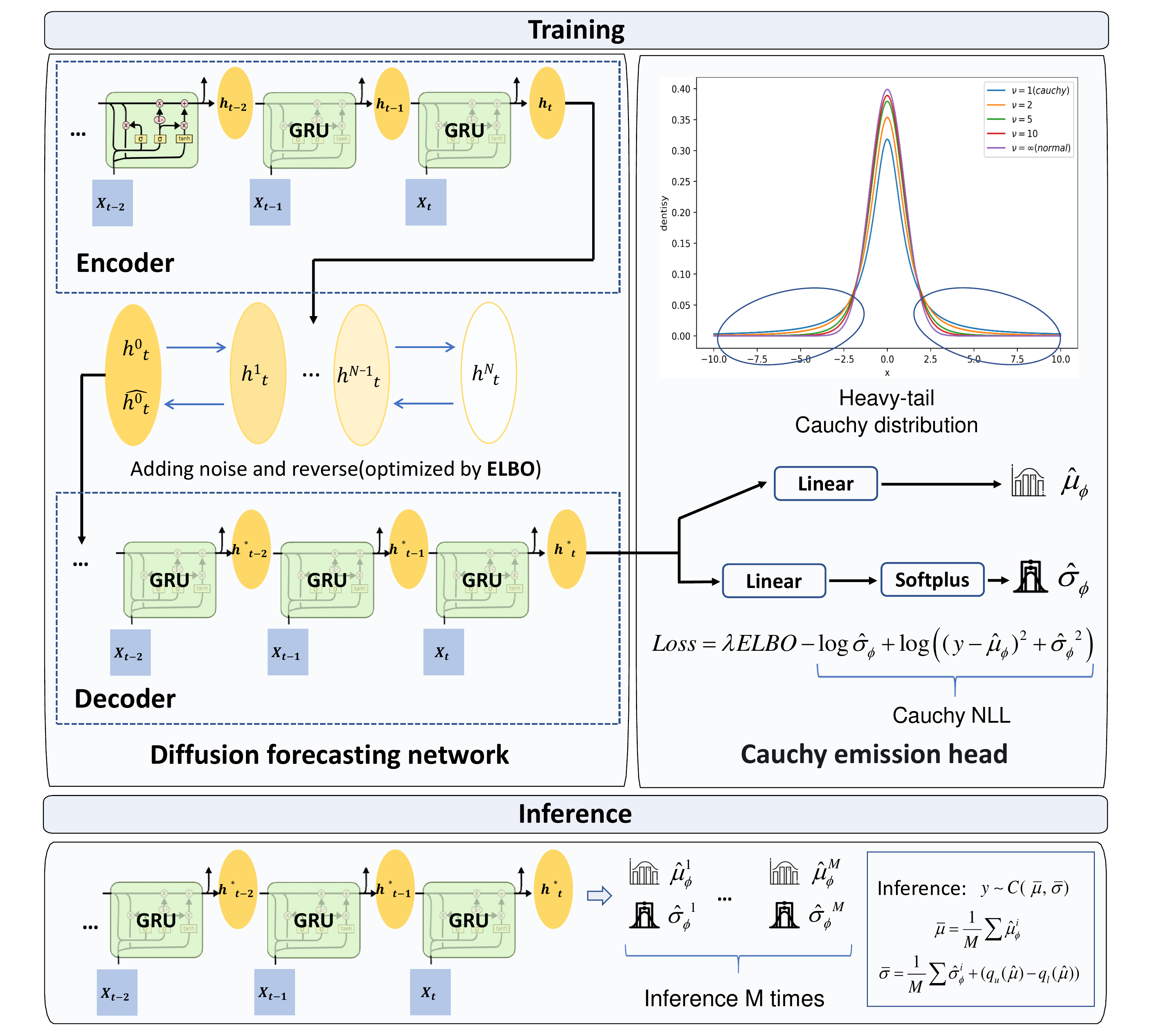}
\centering
\caption{Overview of our proposed DiffLoad method.}
\label{fig2}
\end{figure}

Fig. \ref{fig2} depicts the overall framework of the proposed diffusion forecasting model, including training and inference phases. The proposed Diffusion Load Forecasting model, called DiffLoad, consists of two parts: one is a diffusion forecasting network based on the Seq2Seq structure; the other is an emission head based on the Cauchy distribution. The first part aims to model the probability distribution of hidden states in NN by employing the diffusion structure. This distribution can be seen as the uncertainty of the model itself, i.e., epistemic uncertainty. The second part employs the Cauchy likelihood to model the load data. Based on the characteristics of Cauchy distribution, we can model the uncertainty in load data while resisting the potential adverse effects of issues like data noise. In the following section, we will provide a detailed introduction to the implementation details of these two components and demonstrate how to combine these two components for uncertainty quantification.

\subsection{Epistemic Uncertainty Quantification}
The diffusion model is widely used in generative tasks, which require modeling the data distribution for sampling purposes. However, it is challenging to model the desired distribution without prior knowledge and assumptions. A common approach to address this issue is to transform the desired distribution into a standard distribution and then perform a quasi-transformation. Similar to Variational Auto Encoder (VAE)~\cite{dai2019diagnosing} and Normalizing Flow~\cite{papamakarios2021normalizing}, the diffusion model starts with a normal distribution and eventually transforms it into the desired distribution. Rather than directly transforming the distribution, it proposes a Markov process that breaks the transformation into several steps and adds noise to the original data at each step~\cite{ho2020denoising}. Inspired by the TimeGrad~\cite{rasul2021autoregressive}, which utilizes diffusion structure to generate probabilistic results based on the autoregressive model, we propose to model the distribution of the hidden state of the NN model by the diffusion model. In our model, we transform the hidden state of our Seq2Seq models instead of the original data. This approach allows us to focus on the uncertainties of the model within the hidden state and quantify them. To achieve this goal, our model utilizes the {Gated Recurrent Unit (GRU)~\cite{cho2014learning} as the Encoder, which effectively extracts features from time series data.

Let $q_{\mathbf{h}}\left(\mathbf{h}_{t+1}^0\right)$ denotes the desired distribution of the hidden state. and let $p_{\theta}\left(\mathbf{h}_{t+1}^0\right)$ denote the distribution we use to approximate the real distribution $q_{\mathbf{h}}\left(\mathbf{h}_{t+1}^0\right)$. In the diffusion model, we achieve the approximation by first adding noise $\mathbf{\epsilon}$ to the hidden state: 
\begin{align} 
\mathbf{h}^0_{t+1} &= \text{GRU}(X_{t+1},\mathbf{h}_{t}),\\
\mathbf{h}^{n+1}_{t+1} &= \sqrt{\alpha_n}\mathbf{h}^{n}_{t+1}+\sqrt{1-\alpha_n} \mathbf{\epsilon}, \mathbf{\epsilon} \sim \mathcal{N}\left(\mathbf{0}, \mathbf{I}\right),
\end{align} 
where the $\{\alpha_n\}_{n=1:N} \in (0,1)$ are set values the same as \cite{ho2020denoising}. Note that the Gaussian distribution is stable and has additivity so that we can get the relationship between the original hidden state and the hidden state after $N$ adding noise steps directly.
\begin{align} 
\mathbf{h}^{N}_{t+1} &= \sqrt{\bar{\alpha}_N}\mathbf{h}^{0}_{t+1}+\sqrt{1-\bar{\alpha_N}}\epsilon, \epsilon \sim \mathcal{N}\left(\mathbf{0}, \mathbf{I}\right)\label{for3},
\end{align} 
where $\bar{\alpha}_N=\prod_{n=1}^N \alpha_n$. From \eqref{for3}, we can see that the diffusion process is a kind of interpolation, which makes the original data gradually become white noise. In the following part, we need to figure out how to reverse this process. Note that we will omit the subscripts representing time points without causing ambiguity.

By adding noise, we break the approximation of the desired distribution into several parts $q_{\mathbf{h}}\left(\mathbf{h}^0\right):=\int q_{\mathbf{h}}\left(\mathbf{h}^{0: N}\right) \mathrm{d} \mathbf{h}^{1: N}$ so that we can forecast the desired distribution step by step. However, this kind of breaking, which is  denoted as
\begin{align} 
q_{\mathbf{h}}\left( \mathbf{h}^0\mid \mathbf{h}^{1: N}\right)=\prod_{n=1}^N q_{\mathbf{h}}\left(\mathbf{h}^{n-1} \mid \mathbf{h}^{n}\right),
\end{align} 
is not trainable. To transform the white noise into the hidden state we want, we define a reverse process $p_{\theta}\left(\mathbf{h}^0\right)$ modeled by parameters $\theta$. Similarly, we can break the joint distribution: 
\begin{align}
p_\theta\left(\mathbf{h}^{0:N}\right):=p\left(\mathbf{h}^{N}\right) \prod_{n=1}^N p_\theta\left(\mathbf{h}^{n-1} \mid \mathbf{h}^{n}\right),
\end{align}
where $p\left(\mathbf{h}^{N}\right)$ is assumed to be the standard Gaussian distribution and the other parts are given by a parametrization of our choosing, denoted by
\begin{align}
p_\theta\left(\mathbf{h}^{n-1} \mid \mathbf{h}^{n}\right):=\mathcal{N}\left(\mathbf{h}^{n-1} ; \boldsymbol{\mu}_\theta\left(\mathbf{h}^{n}, n\right), \mathbf{\Sigma}_\theta\left(\mathbf{h}^{n}, n\right)\right)\label{for6}.
\end{align}
With these preparations, we establish a model to eliminate the Gaussian noise by minimizing the negative log-likelihood (NLL):
$-\log p_\theta\left(\mathbf{h}^{0}\right)$. Note that we cannot get the closed form of the reverse distribution $q_{\mathbf{h}}\left(\mathbf{h}^{n-1} \mid \mathbf{h}^{n}\right)$. However, we can fix this problem by considering the origin hidden state as a condition and then minimizing Evidence Lower Bound (ELBO)~\cite{blei2017variational}, which is the upper bound of the NLL.
\begin{align}
&-\log p_\theta\left(\mathbf{h}^{0}\right)= - \log \int p_\theta\left(\mathbf{h}^{0:N}\right) d \mathbf{h}^{1:N}
 \\
& \leq \underbrace{\mathbb{E}_{q\left(\mathbf{h}^{1:N} \mid \mathbf{h}^{0}\right)}\left[\log \frac{p_\theta\left(\mathbf{h}^{0:N}\right)}{q_{\mathbf{h}}\left(\mathbf{h}^{1:N} \mid \mathbf{h}^{0}\right)}\right]}_{ELBO}.
\end{align}
By adding the origin hidden state as the condition, we can rewrite the reverse process like this 
\begin{align}
&q_{\mathbf{h}}\left(\mathbf{h}^{n-1} \mid \mathbf{h}^{n}, \mathbf{h}^{0}\right)\\
&=q_{\mathbf{h}}\left(\mathbf{h}^{n} \mid \mathbf{h}^{n-1}, \mathbf{h}^{0}\right) \frac{q_{\mathbf{h}}\left(\mathbf{h}^{n-1} \mid \mathbf{h}^{0}\right)}{q_{\mathbf{h}}\left(\mathbf{h}^{n} \mid \mathbf{h}^{0}\right)}\nonumber,\\
& \propto \exp \biggl(-\frac{1}{2}\biggl(\left(\frac{\alpha_n}{\beta_n}+\frac{1}{1-\bar{\alpha}_{n-1}}\right) (\mathbf{h}^{n-1})^2\nonumber\\ 
&-\left(\frac{2 \sqrt{\alpha_n}}{\beta_n} \mathbf{h}^{n-1}+\frac{2 \sqrt{\bar{\alpha}_{n-1}}}{1-\bar{\alpha}_{n-1}} \mathbf{h}^{0}\right) \mathbf{h}^{n-1}+C\biggr)\biggr),
\end{align}
where $\beta_n = 1-\alpha_n$, and this formula can be transformed in the form of Gaussian density
\begin{align}
    q_{\mathbf{h}}\left(\mathbf{h}^{n-1} \mid \mathbf{h}^{n}, \mathbf{h}^{0}\right) \propto \mathcal{N}\left(\mathbf{h}^{n-1} ; \tilde{\boldsymbol{\mu}}\left(\mathbf{h}^{n}, \mathbf{h}^{0}\right), \tilde{\beta}_n \mathbf{I}\right).
\end{align}
where
\begin{align}
\tilde{\beta}_n=1 &=\frac{1-\bar{\alpha}_{n-1}}{1-\bar{\alpha}_n}\beta_n,\\
\tilde{\boldsymbol{\mu}}_n\left(\mathbf{h}^{n}, \mathbf{h}^{0}\right)&=\frac{\sqrt{\alpha_n}\left(1-\bar{\alpha}_{n-1}\right)}{1-\bar{\alpha}_n} \mathbf{h}^{n}+\frac{\sqrt{\bar{\alpha}_{n-1}} \beta_n}{1-\bar{\alpha}_n} \mathbf{h}^{0}.
\end{align}
Note that $ELBO = \mathcal{L}_0+\sum_{n=2}^N \mathcal{L}_{n-1}+\mathcal{L}_N$,
where
\begin{align}
\mathcal{L}_0& := - \mathbb{E}_{q_{\mathbf{h}}\left(\mathbf{h}^{1} \mid \mathbf{h}^{0}\right)} \log p_\theta\left(\mathbf{h}^{0} \mid \mathbf{h}^{1}\right), \\ 
\mathcal{L}_{n-1}& :=D_{\mathrm{KL}}\left(q_{\mathbf{h}}\left(\mathbf{h}^{n-1} \mid \mathbf{h}^{n}, \mathbf{h}^{0}\right) \| p_\theta\left(\mathbf{h}^{n-1} \mid \mathbf{h}^{n}\right)\right)\label{for14}, \\ 
\mathcal{L}_N & :=D_{\mathrm{KL}}\left(q_{\mathbf{h}}\left(\mathbf{h}^{N} \mid \mathbf{h}^{0}\right) \| p\left(\mathbf{h}^{N}\right)\right).
\end{align}
Recalling from \eqref{for6}, we let $\mathbf{\Sigma}_\theta\left(\mathbf{h}^n, n\right) = \tilde{\beta}_n \mathbf{I}$ for $n = 1,\dots, N-1$ to simplify the training and make the training process smoother. In this way, we can solve the problem of approximating the reverse distribution by approximating the expectation, that is 
\begin{align}
    \eqref{for14}=\mathbb{E}_{q_{\mathbf{h}}\left(\mathbf{h}^{n} \mid \mathbf{h}^{0}\right)}\left[\frac{1}{2 \tilde{\beta}_n}\left\|\tilde{\boldsymbol{\mu}}_n\left(\mathbf{h}^{n}, \mathbf{h}^{0}\right)-\boldsymbol{\mu}_\theta\left(\mathbf{h}^{n}, n\right)\right\|^2\right]\label{for16}.
\end{align}
For $n=0$, \cite{DBLP:journals/corr/abs-2110-05948} claims that we can ignore it for simplification. As for $n=N$, $p\left(\mathbf{h}^{N}\right)$ is the standard Gaussian distribution that no parameters need to be learned. 
So far, we have demonstrated how to construct a reverse distribution to denoise white noise and it seems simple enough to use a NN to forecast the expectation. However, it is worth noting that the expectation here is generated by a non-standard normal distribution, which cannot be processed by simple gradient descent. Therefore, we need to use the re-parameter trick commonly used in VAE models. Recalling from \eqref{for3}, we can replace expectations with noise to rewrite optimization objectives:
\begin{align}
    &\eqref{for16}\nonumber\\ 
    &= \mathbb{E}_{\mathbf{h}^{0}, \epsilon \sim \mathcal{N}(\mathbf{0}, \mathbf{I})}\biggl[\frac{1}{2 \tilde{\beta}_n}\biggl\|\frac{1}{\sqrt{\alpha_n}}\left(\mathbf{h}^{n}\left(\mathbf{h}^{0}, \epsilon\right)-\frac{\beta_n}{\sqrt{1-\bar{\alpha}_n}} \epsilon\right) \notag\\
    &-\boldsymbol{\mu}_\theta\left(\mathbf{h}^{n}\left(\mathbf{h}^{0}, \epsilon\right), n\right)\biggr\|^2\biggr],\\
    & \propto \mathbb{E}_{\mathbf{h}^{0}, \epsilon \sim \mathcal{N}(0, \mathbf{I})}\left\|\epsilon-\epsilon_\theta\left(\sqrt{\bar{\alpha}_n} \mathbf{h}^{0}+\sqrt{1-\bar{\alpha}_n} \epsilon, n\right)\right\|^2\label{for18}.
\end{align}
Note that we removed the weight coefficients in the final simplification step to obtain the final optimization goal. By reducing the difference between the real generated Gaussian noise and the noise generated by the NN, we can use the $p_\theta\left(\mathbf{h}^{n-1} \mid \mathbf{h}^n\right)$ to approximate the $q_{\mathbf{h}} \left(\mathbf{h}^{n-1} \mid \mathbf{h}^n, \mathbf{h}^0\right)$ step by step, thus transforming the white noise to a probabilistic hidden state, which represents the epistemic uncertainty of the model. We will illustrate the detailed process for training and inference in Subsection \ref{train_and_infer}.

\subsection{Aleatoric Uncertainty Quantification}\label{aleatoric}
To model aleatoric uncertainty, we will employ an emission head to wrap the forecasting model. The emission head controls the conditional error distribution between the labels and forecasts. Instead of using the traditional Gaussian distribution which is not heavy-tailed, we suggest using the heavy-tailed Cauchy distribution to make the model more robust to outliers and mutation according to robust statistics~\cite{huber2011robust,li2022learning}. Similar to the Gaussian distribution, the Cauchy distribution can be modeled by location and scale parameters:
\begin{align}
f\left(y ; \mu, \sigma\right)=\frac{1}{\pi \sigma\left[1+\left(\frac{y-\mu}{\sigma}\right)^2\right]}=\frac{1}{\pi}\left[\frac{\sigma}{\left(y-\mu\right)^2+\sigma^2}\right].
\end{align}

To demonstrate how the Cauchy distribution is robust and tolerates the noise on standardized training labels $\{y_t\}_{t=1}$, we first define three common types of training label noise:
\begin{enumerate}
    \item Constant type noise: $y_t^{\mathrm{A}}=y_t+\epsilon, \epsilon$ is constant.
    \item Missing type noise: $y_t^{\mathrm{A}}=\epsilon, \epsilon$ is constant. 
    \item Gaussian type noise: $y_t^{\mathrm{A}}=y_t+\epsilon, \epsilon \sim \mathcal{N}\left(0, \sigma^2\right)$.
\end{enumerate}

Note that $\{y_t\}_{t=1}$ represent clean labels without noise. In observed training labels $\{\tilde y_t\}_{t=1}$, we assume that whether there is noise depends on a binomial distribution, which means 
$$\widetilde{y_t}=\left\{\begin{array}{ll}y_t, & \text { with probability } 1-\eta, \\ y_t^{\mathrm{A}}, & \text { with probability } \eta.\end{array}\right.$$

\begin{lemma}\label{lemma1}~\cite{cheng2024robusttsf} Let $\ell$ be the loss function, $f$ be the forecasting model, $R_{\ell}(f)$ be the empirical loss on the clean training set, and $R^\eta_{\ell}(f)$ be the empirical loss on the training set with noise. Under different noise anomalies with anomaly rate $\eta<0.5$, we have
\begin{align}
R^\eta_{\ell}(f)=(1-2 \cdot \eta) \cdot R_{\ell}(f)+\eta \cdot \mathbb{E}_{\boldsymbol{x}}[ \ell\left(f(\boldsymbol{x}), y\right)+\ell(f(\boldsymbol{x}), y^A)]
\end{align}
\end{lemma}

Lemma \ref{lemma1} suggests that the empirical loss on clean labels can be represented as an affine variation of the empirical loss with noise. Therefore, the impact of the label noise on the loss function (where we use negative log-likelihood with different probability distributions) will mainly be reflected in $\ell\left(f(\boldsymbol{x}), y\right)+\ell(f(\boldsymbol{x}), y^A)$. 

\begin{theorem}\label{theorem1}
For outliers that exceed the predicted scale, i.e., $| y^A - f(\boldsymbol{x}) | \geq \sigma$, we have 
\begin{align}
|\frac{d R_{c}(f)}{d y^A}| \leq |\frac{d R_{g}(f)}{d y^A}|,
\end{align}
where $R_{c}(f)$ and $R_{g}(f)$ represent the empirical loss of using Cauchy likelihood and Gaussian likelihood, respectively.
\end{theorem}

\begin{proof}
According to Lemma \ref{lemma1}, we only need to analyze $\ell\left(f(\boldsymbol{x}), y\right)+\ell(f(\boldsymbol{x}), y^A)$.
Substitute the negative log-likelihood functions of the Cauchy distribution and Gaussian distribution separately and then take the derivative of the noisy label, we have
\begin{align}|\frac{d R_{c}(f)}{d y^A}| &= \frac{2|y^A-f(\boldsymbol{x})|}{(f(\boldsymbol{x})-y^A)^2+\sigma^2}\\
 |\frac{d R_{g}(f)}{d y^A}| &= \frac{|y^A-f(\boldsymbol{x})|}{\sigma^2}\\
|\frac{d R_{c}(f)}{d y^A}| / |\frac{d R_{g}(f)}{d y^A}| &= \frac{2\sigma^2}{(f(\boldsymbol{x})-y^A)^2+\sigma^2}\leq 1
\end{align}
\end{proof}
Theorem \ref{theorem1} indicates that the usage of Cauchy distribution has a higher tolerance for outlier labels during training, and the variation of outlier labels does not greatly affect the performance of the model in fitting clean labels. In contrast, Gaussian distribution is more susceptible to the influence of outliers, and slight perturbations may cause significant changes in the model, which is not conducive to learning with noisy labels.

In our implementation, the parameters of the Cauchy distribution will be given by the Decoder parameterized by $\phi$ like this (To avoid confusion with the input of the encoder, we mark * above to indicate the input of the decoder)
\begin{align}
\mathbf{h}_{t+1}^* &= \text{GRU}(X_{t},\mathbf{h}_{t}^{*}),\\
p_{\phi}\left(X_{t+1} \mid \mathbf{h}^*_{t+1}\right)&=\mathcal{C}\left(X_{t+1}; \boldsymbol{\mu}_{\phi(t+1)}, \boldsymbol{\sigma}_{\phi(t+1)}\right),
\end{align}
where 
\begin{align}
\boldsymbol{\mu}_{\phi(t+1)} & =\mathrm{Linear}_1\left(\mathbf{h}_{t+1}^*\right), \\ \boldsymbol{\sigma}_{\phi(t+1)} & =\operatorname{SoftPlus}\left[\mathrm{Linear}_2\left(\mathbf{h}_{t+1}^*\right)\right].
\end{align}

In this way, we can model the conditional distribution of the error by the Cauchy distribution, which represents the aleatoric uncertainty. Note that the Cauchy distribution is a special case of Student-T distribution, as shown in Fig. \ref{fig2}. Although some degrees of flexibility are sacrificed, the advantage of the Cauchy distribution is that it is a $\alpha$-stable distribution.

\begin{myDef}\cite{georgiou1999alpha}
$\alpha$-stable distribution is a kind of distribution that has no general closed form, but it can be defined by the continuous Fourier transform of its characteristic function $\varphi(t)$,
\begin{align}
f(x ; \alpha, \beta, \sigma, \mu)=\frac{1}{2 \pi} \int_{-\infty}^{+\infty} \varphi(t) e^{-i t x} d t,
\end{align}
\begin{align}
    \varphi(t)=\exp \left[i t \mu-|\sigma t|^\alpha(1-i \beta \operatorname{sgn}(t) \Phi)\right],
\end{align}
where
\begin{align}
    \begin{split}
\Phi= \left \{
\begin{array}{ll}
    -(2 / \pi) \log |t|,                    & \alpha=1,\\
    \tan (\pi \alpha / 2),                                 & otherwise.
\end{array}
\right.
\end{split}
\end{align}
\end{myDef}
Note that Gaussian and Cauchy are both special cases of this kind of distribution while Student-T is not. The advantage of using $\alpha$-stable distribution here is that we can combine two kinds of uncertainties by their additivity and linear transformation invariance.
\begin{lemma}\label{lemma2}
Suppose $X_1$ and $X_2$ are two random variables that subject to 
$\alpha$-stable distribution $f(x ; \alpha, \beta, \sigma_1, \mu_1)$ and $f(x ; \alpha, \beta, \sigma_2, \mu_2)$, then we have 
\begin{align}
    aX_1+b &\sim f(x ; \alpha, \beta, a^{\alpha}\sigma_1, a\mu_1+b),\\
    X_1+X_2 &\sim f(x ; \alpha, \beta, (\sigma_1^{\alpha}+\sigma_2^{\alpha})^{\frac{1}{\alpha}}, \mu_1+\mu_2).
\end{align}
\end{lemma}
With this property, we can estimate two kinds of uncertainties separately and combine the uncertainties by adding them together directly, which will be described in the following section.

\subsection{Training and Inference}\label{train_and_infer}
As shown in Fig. \ref{fig2}, we first get the hidden state $\hat{\mathbf{h}}_{t+1}^0$ after inputting the data into the diffusion-based Encoder. During this process, we concentrate the uncertainty of the model into the hidden state. This idea was inspired by \cite{tagasovska2019single}. Based on a fixed network structure, we can evaluate whether a neural network has confidence in its hidden state by observing the degree of change in the network's hidden. Unlike \cite{tagasovska2019single}, we do not need to construct explicit constraints on the network's hidden state and train multiple models. Instead, we directly use the diffusion model to model the probabilistic hidden state. Through probabilistic modeling results, we can obtain the confidence level of the model in the hidden state and thus estimate the epistemic uncertainty. For the diffusion model, we have the first term of the optimization goal, which is stated in \eqref{for18}. Then, we put the estimated hidden state into the Decoder. The output of the Decoder will be seen as the parameter of the emission distribution and optimized by the NLL. During the training process, We combine the two losses through hyperparameter $\lambda$:
\begin{align}\label{for29}
    \mathcal{L} = \lambda \cdot ELBO -\log \hat{\sigma}_{\phi}+\log \left((y-\hat{\mu}_\phi)^2+\hat{\sigma}_{\phi}^2\right)
\end{align}

\begin{algorithm}[htb]
\caption{Training process of diffusion-based Seq2Seq DiffLoad method.}\label{ag1}
\KwData{Encoder $P_{\phi}$;
Decoder $Q_{\phi}$;
Reverse network $R_{\theta}$
Input load history data $X$; label $Y$}
\While{not convergence}
{
\For{$\text{batch}$ in $\text{batch loader}$}{
$\mathbf{h}^0\leftarrow P_{\phi}(X)$\;
Add Gaussian to the hidden state for N steps,
$\mathbf{h}^{N} = \sqrt{\bar{\alpha_N}}\mathbf{h}^{0}+\sqrt{1-\bar{\alpha_N}}\epsilon, \epsilon \sim \mathcal{N}\left(\mathbf{0}, \mathbf{I}\right)$;\

Sample $n \sim U(\{1,2,3 \ldots N\})$\;
Sample $\epsilon \sim \mathcal{N}(\mathbf{0}, I)$\;
Calculate the ELBO loss,
$ELBO_{\epsilon}=\left\|\epsilon-\epsilon_\theta\left(\sqrt{\bar{\alpha}_n} \mathbf{h}^0+\sqrt{1-\bar{\alpha}_n} \epsilon, n, R_\theta\right)\right\|^2$\;
Reconstruct the hidden state $\mathbf{h}^*;$\\
$\hat{\sigma}_{\phi}, \hat{\mu}_{\phi}\leftarrow Q_{\phi}(X,\mathbf{h}^*);$\\
Combine the ELBO with the NLL with hyperparameter $\lambda$,
$\mathcal{L}=(\ref{for29})$\;
Take gradient descent step on $\nabla_{\theta,\phi}\mathcal{L}$\;

}
}
\KwResult{Encoder $P_{\phi}$;
Decoder $Q_{\phi}$;
Reverse network $R_{\theta}$}
\end{algorithm}
As for the inference process, we will infer for $M$ times with all other settings consistent. In each inference process, the output of the Encoder undergoes the process of adding and removing noise, thus exhibiting randomness like most of the deep state model~\cite{salinas2020deepar}. The output of our model is the parameters of the emission model and the location parameter will be the average of multiple inferences.
\begin{align}
    \bar{\mu}=\frac{1}{M} \sum \hat{\mu}_\phi^i
\end{align}

In terms of uncertainty estimation, we separate the two kinds of uncertainties. On the one hand, the scale parameters provided by the model represent aleatoric uncertainty, and on the other hand, the distance between upper and lower quantiles of location parameters obtained through multiple inferences represents epistemic uncertainty. For the probabilistic forecasting, we use Lemma \ref{lemma1} to directly add the two uncertainties together.
\begin{align}
\centering
    \bar{\sigma}&=\hat{\sigma}_{\phi}+ \hat{\sigma}_{\theta},\\
    & = \frac{1}{M} \sum \hat{\sigma}_\phi^i+\left(q_u(\hat{\mu})-q_l(\hat{\mu})\right)
\end{align}
where $\hat{\sigma}_{\phi}$, $\hat{\sigma}_{\theta}$ represents the estimated epistemic and aleatoric uncertainty seperately and $q_u()$ and $q_l()$ are the upper and and lower quantile of the samples $\{\hat{\mu}_{\phi}^1, \hat{\mu}_{\phi}^2, \dots, \hat{\mu}_{\phi}^M\}$.

\begin{algorithm}[htb]
\caption{Inference process of diffusion-based Seq2seq DiffLoad method.}\label{ag2}
\KwData{Encoder $P_{\phi}$;
Decoder $Q_{\phi}$;
Reverse network $R_{\theta}$
Input load history data $X$}
\For{$m=M$ to $1$}{
$\mathbf{h}^0\leftarrow P_{\theta}(X)$\;
$\mathbf{h}^{N} = \sqrt{\bar{\alpha_N}}\mathbf{h}^{0}+\sqrt{1-\bar{\alpha_N}}\epsilon, \epsilon \sim \mathcal{N}\left(\mathbf{0}, \mathbf{I}\right)$;\

\For{$n=N$ to $1$}{
\eIf{$n > 1$}{
Draw $z \sim \mathcal{N}(\mathbf{0}, I)$
}{else set $z = 0$}
Calculate re-parameterized term $\mathbf{h}^{n-1}=\frac{1}{\sqrt{\alpha_n}}\left(\mathbf{h}^n-\frac{\beta_n}{\sqrt{1-\bar{\alpha}_n}} \epsilon_\theta\left(\mathbf{h}^n, n, R_\theta\right)\right)+\sqrt{\tilde{\beta}_t} \mathbf{z}$;\
}
$\hat{\sigma}_{\phi}^m, \hat{\mu}_{\phi}^m\leftarrow Q_{\phi}(X,\mathbf{h}^*)$;\
}
\KwResult{$\{\hat{\sigma}_{\phi}^1, \hat{\sigma}_{\phi}^2, \dots, \hat{\sigma}_{\phi}^M\}$, $\{\hat{\mu}_{\phi}^1, \hat{\mu}_{\phi}^2, \dots, \hat{\mu}_{\phi}^M\}$}
\end{algorithm}

With the labels used above, we summarize the training and inference process of our framework in Algorithm \ref{ag1}, and Algorithm \ref{ag2}, respectively.

\section{Case Studies}\label{Example analysis}
\subsection{Experimental Setups}
In this section, we use three data sets to verify the effectiveness of our method. 
\begin{itemize}
    \item Global Energy Forecasting (GEF) competition~\cite{hong2016probabilistic}. It contains the power load data from 2004 to 2014. The data from 2012 to 2014 is used to verify our model.
    \item The BDG2 dataset~\cite{miller2020building}. This data set contains energy data for 2 years (from 2016 to 2017) from 1,636 buildings. We randomly selected 10 buildings with different usages (e.g., education, lodging, and industrial) from it.
    \item The dataset from Day-ahead electricity demand forecasting competition: Post-covid paradigm~\cite{farrokhabadi2022day}. This data set
contains energy data from 2017-03-08 to 2020-11-06. As shown in Fig. \ref{COV}, the power load data has significantly deviated from the original pattern after the outbreak of COVID-19.
\end{itemize}

Since our model can give both deterministic and probabilistic forecasts. Mean Absolute Percentage Error (MAPE) is used to evaluate the deterministic forecasts. To avoid the drawbacks of MAPE which is sensitive to the value near 0 (even though it is unusual in load forecasting), we also add the Mean Absolute Error (MAE) metric as a supplement. 
Continuous Ranked Probability Score (CRPS)~\cite{matheson1976scoring} and Winker Score~\cite{wang2022adaptive} are used to evaluate the probabilistic forecasts. For comparison baselines, we introduce two MLP-based time series forecasting methods for deterministic forecasting and five probabilistic methods for uncertainty quantification.

\begin{itemize}

\item \text{NBEATS}~\cite{oreshkin2019n}: A time series forecasting method based on MLP, which can select blocks with different structures to decompose the sequence. Here we compared two types of NBEATS. One is based on two Generic Blocks (denoted as G), and the other is based on a Trend Block and a Seasonality Block (denoted as TS). Note that here we consider the influence of covariates and use the version from~\cite{olivares2023neural}, denoted as NBEATSX.

\item \text{DLinear}~\cite{zeng2023transformers}: A time series forecasting method based on MLP, which performs trend period decomposition on time series. We added a linear layer to the model to map multivariate sequences containing covariates to univariate sequences.

\item \text{GRU}~\cite{cho2014learning}: a kind of RNN structure widely used in sequence modeling. All the methods mentioned below are based on the GRU structure. Here, we apply the forecasting error of GRU on the training set to construct probabilistic forecasting and quantify uncertainty~\cite{liang2023uncertainty}.

\item \text{DeepAR}~\cite{salinas2020deepar}: a kind of deep state model, which uses the normal distribution to model the output of the deep neural network. In this way, we can model the aleatoric uncertainty by the normal distribution.

\item \text{Deep Ensemble}~\cite{lakshminarayanan2017simple}: The methods mentioned above only consider the aleatoric uncertainty caused by the data itself, while ignoring the uncertainty introduced by the neural network. Ensemble methods address this by training multiple models during the training process. During the testing process, we combine the outputs of all networks and assume their output conforms to a normal distribution, similar to the GRU with MSE loss functions.

\item \text{Bayes By Backprop(BBB)}~\cite{tran2019bayesian,kendall2017uncertainties}: Bayesian neural network is an extension of the ensemble training method. Rather than training multiple networks, Bayesian neural networks consider the parameters in neural networks as random variables instead of fixed values. Because of the excellent conjugate property of normal distribution, such methods usually assume that the prior distribution of parameters is normal, and use the reparameterization tricks to establish ELBO so that it can be optimized by gradient descent methods.

\item \text{MC dropout (MCD)}~\cite{gal2016dropout}: MC dropout has been proven to be an approximation of Bayesian neural networks while maintaining a much lower computational cost. During the training process, we set the dropout layer in the same manner as in ordinary settings. However, in the inference process, we do not turn off the dropout layer, allowing us to obtain probabilistic output. We then consider the output to correspond to a normal distribution.

\end{itemize}

Methods based on GRU networks share the same hyperparameter shown in Table \ref{hyper}. Specifically, for hidden size, we choose from \{32,64\}, and then we test the performance of the original GRU on the GEF dataset. For hidden layers and bias, we did the same and chose the best setting from \{1,2\} and \{True, False\}, respectively. And the hidden size of NBEATS class methods is also 64. Note that all other methods do not enable dropout except for the dropout method, which requires adjusting the dropout rate to 0.25 for probability output (the same in~\cite{liang2023uncertainty}). In addition to consistent model parameters, all methods use the same training process, where the batch size is 256, and we use Adam as the optimizer with an initial learning rate of 5e-3. During training, we set the hyperparameter $\lambda$ to 1 and adopted an early stop mechanism. If 15 consecutive gradient updates do not achieve better RMSE on the validation set or if the total number of training epochs reaches 300, we will stop continuing the training. In terms of probability inferencing, we set the number of times M for repeated inference to 100. At the same time, our method will search the validation set, and select the quantile distance of the lowest CRPS from the quantile distances of $10\%$, $30\%$, $50\%$, and $70\%$ as our estimation of epistemic uncertainty.

\begin{table}[htb]
\centering
\caption{Hyperparameters of the GRU module}
\begin{tabular}{cccc}
\hline\hline
hidden size & hidden layers & bias & dropout rate \\ \hline
(64,64)     & 2             & True & 0(0.25)      \\ \hline\hline
\end{tabular}
\label{hyper}
\end{table}

\subsection{Experimental Results}
In this section, we will compare the performance of different uncertainty estimation models on multiple datasets. Among them, the dataset GEF is a relatively stable aggregated level load. COV is also aggregated level data, but when COVID-19 comes, it shows an obvious deviation. BDG2 includes 10 different types of buildings, some of which may have a large number of outliers and bad data, making data cleaning more difficult. To demonstrate the robustness of the Cauchy emission head in combating outliers and offsets, we will not perform additional processing on it.

\subsubsection{Results Analysis and Discussion}
Table \ref{MAPE} and Table \ref{MAE} summarize the deterministic forecasting result for different methods. Our method has defeated all competitive baselines. In addition to directly comparing the accuracy of different models, we can also draw several conclusions from this experimental result. Firstly, the simple GRU network achieved the worst performance in almost all tasks. This indicates that estimating uncertainty is not only beneficial for probabilistic forecasting, but also has a positive effect on deterministic forecasting. Similarly, models that consider the uncertainty brought by the model itself, such as Ensemble, Bayesian, and other methods, are also superior to the DeepAR model that only considers aleatoric uncertainty in most cases. Indicating that epidemic uncertainty is indeed a factor worth considering in the training process of deep neural networks. Secondly, from the perspective of model structure, our seq2seq structure based on Gaussian distribution emission heads achieved suboptimal results in the average of 10 building datasets and the GEF dataset. However, the Gaussian emission head method yields poor results in non-stationary datasets (COV). Compared with the method based on Cauchy emission heads using the same structure, the Gaussian distribution has significant shortcomings in dealing with data mutations. Lastly, for different types of epistemic uncertainty quantification methods, the dropout method lags behind our method by 5.5\%, 21\%, and 8.7\% on three datasets, respectively. Especially in COV datasets, such differences are even more pronounced. The possible reason for this is that the dropout method undermines the consistency between model training and inference. The strong regularization effect brought by the dropout method makes it difficult for the model to learn useful knowledge from data. While compared with other methods such as the Bayesian method, our method can significantly reduce the required computational time while maintaining performance advantages (shown in Section \ref{time_eval}). As for the MLP-based forecasting method, they generally can not provide competitive performance when considering external variables, especially DLinear. The main reason for this is that these time series models are mainly designed to model the temporal characteristic of long sequences, and their performance may be poor when external variables such as temperature need to be considered. Even though NBEATSX, which considers external variables to some degree, performs relatively well in the BDG2 dataset, it still lags behind our methods by about 12\% in the MAPE metric.

\begin{table}[htb]\Huge
\centering
\caption{MAPE comparison}
\renewcommand{\arraystretch}{1.7}
\setlength{\tabcolsep}{6pt}
\resizebox{0.5\textwidth}{!}{
\begin{tabular}{ccccccccc}
\hline\hline
                  & \multicolumn{8}{c}{MAPE}                                                                                                                                                                           \\ \hline
                  & GRU                 & deepAR              & MC dropout              & Bayesian             & Ensemble    & DLinear & NBEATSX-G(TS)                 & Ours Cauchy(Gaussian)         \\ \hline
GEF               & 3.49                & 3.58                & 3.48                 & 3.48                 & 3.44                     & 4.21 & 3.86(3.85)              & \textbf{3.29}(3.44)        \\
COV               & 2.34                & 2.04                & 2.44                 & 1.96                 & 1.98                     & 3.51 & 2.46(2.59)            & \textbf{1.91}(2.13)        \\
BDG2              & 12.61               & 12.46               & 12.13                & 12.30                & 12.08                  & 13.17 & 12.38(12.33)               & \textbf{10.83}(11.95)       \\ \hline\hline
\end{tabular}
}\label{MAPE}
\end{table}

\begin{table}[htb]\Huge
\centering
\caption{MAE comparison}
\renewcommand{\arraystretch}{1.7}
\setlength{\tabcolsep}{6pt}
\resizebox{0.5\textwidth}{!}{
\begin{tabular}{ccccccccccc}
\hline\hline
                  & \multicolumn{8}{c}{MAE}                                                                                                                                                                           \\ \hline
                  & GRU                 & deepAR              & MC dropout              & Bayesian             & Ensemble  & DLinear & NBEATSX-G(TS)                & Ours Cauchy(Gaussian)         \\ \hline
GEF               & 117.15              & 119.47              & 116.98               & 116.29               & 115.30               & 142.72   & 130.59(129.90)       & \textbf{110.44}(115.37)      \\
COV               & 24859.30            & 21844.09            & 25821.99             & 20806.48             & 21044.39         & 39292.92 & 27358.27(28819.07)                  & \textbf{20308.39}(22512.37)    \\
BDG2              & 12.82               & 12.76               & 12.25                & 12.74                & 12.48             & 12.85 & 12.28(12.31)              & \textbf{11.18}(12.38)       \\ \hline\hline
\end{tabular}
}\label{MAE}
\end{table}

Table \ref{prob_result} summarizes the probabilistic forecasting result for different methods. Among them, $25\%$, $50\%$, and $75\%$ of the Winker scores evaluated the probability estimates for non-conservative, general, and conservative situations, respectively. At the same time, we also used CRPS to evaluate the overall performance of probability estimation. From the results, our method maintains its advantages in most cases, only with a reduction of $0.38\%$ compared to the ensemble-based method when calculating CRPS metrics on the GEF dataset. From the perspective of the Winker Score, our method exhibits superior performance compared to other baselines in conservative, general, and non-conservative situations. In addition to the overall situation, Fig. \ref{building} also compares the forecasting accuracy of 10 building datasets. It can be seen that our method can provide advanced forecasting results and achieve significant improvements in average results, except for a few datasets with larger absolute load values. Fig. \ref{prob} provides a comparison between our method and the probability forecasts provided by other methods with a $75\%$ confidence interval. It can be seen that the $75\%$ interval is sufficient to cover the actual load value. However, compared to our method, the intervals given by other methods are slightly wider, indicating that other methods are too conservative and lead to a decrease in forecasting accuracy.

\begin{table*}[t]\Huge
\centering
\caption{Probabilistic result comparison}
\renewcommand{\arraystretch}{1.7}
\setlength{\tabcolsep}{6pt}
\resizebox{\textwidth}{!}{
\begin{tabular}{cccccccccccccc}
\hline\hline
                                    & \multicolumn{8}{c}{Metric}                                                                     & \multicolumn{5}{c}{Relative Improvement(\%)}                  \\ \cline{2-14} 
                                    & Dataset & GRU      & deepAR   & MC dropout & Bayesian & Ensemble & Ours Gaussian & Ours Cauchy & I\_GRU   & I\_deepAR & I\_dropout & I\_Bayesian & I\_ensemble \\ \hline
\multirow{3}{*}{Winker score(25\%)} & GEF     & 74.84    & 75.88    & 74.24      & 74.00    & 73.22    & 73.39        & \textbf{70.69}        & +5.54\%  & +6.83\%   & +4.78\%    & +4.47\%     & +3.45\%     \\
                                    & COV     & 15989.00  & 13966.90  & 16576.91   & 13310.35 & 13441.33  & 14398.77     & \textbf{13063.04}     & +18.30\% & +6.47\%   & +21.19\%   & +1.85\%     & +2.81\%     \\
                                    & BDG2    & 8.21     & 8.14     & 7.91       & 8.15     & 8.05     & 7.93         & \textbf{7.17}         & +12.66\% & +11.91\%  & +9.35\%    & +12.02\%    & +10.93\%    \\ \hline
\multirow{3}{*}{Winker score(50\%)} & GEF     & 96.76    & 96.66    & 94.38      & 94.08    & 93.16    & 93.22        & \textbf{90.94}        & +6.01\%  & +5.91\%   & +3.64\%    & +3.33\%     & +2.38\%     \\
                                    & COV     & 20975.60  & 18069.30  & 21758.42   & 17185.82 & 17334.50 & 18601.00     & \textbf{17112.44}     & +18.41\% & +5.29\%   & +21.35\%   & +0.43\%     & +12.81\%    \\
                                    & BDG2    & 10.68    & 10.48    & 10.44      & 10.56    & 10.38    & 10.28        & \textbf{9.36}         & +12.35\% & +10.68\%  & +10.34\%   & +11.36\%    & +9.82\%     \\ \hline
\multirow{3}{*}{Winker score(75\%)} & GEF     & 133.10   & 131.00   & 127.17     & 125.11   & 124.64   & 124.18       & \textbf{121.03}       & +9.06\%  & +7.61\%   & +4.82\%    & +3.26\%     & +2.89\%     \\
                                    & COV     & 29367.07 & 25408.52 & 31361.84   & 24022.10 & 23885.87 & 25988.62     & \textbf{23399.30}     & +20.32\% & +7.90\%   & +25.38\%   & +25.92\%    & +2.03\%     \\
                                    & BDG2    & 14.82    & 14.39    & 14.71      & 14.55    & 14.36    & 14.26        & \textbf{13.60}        & +8.26\%  & +5.48\%   & +7.54\%    & +6.52\%     & +5.29\%     \\ \hline
\multirow{3}{*}{CRPS}               & GEF     & 86.25    & 86.61    & 84.35      & 83.61    & \textbf{82.92}    & 83.07        & 83.24        & +3.48\%  & +3.89\%   & +1.31\%    & +0.44\%     & -0.38\%     \\
                                    & COV     & 18575.06 & 16263.42 & 19585.95   & 15438.46 & 15479.88 & 16705.65     & \textbf{15435.79}     & +16.90\% & +5.08\%   & +21.18\%   & +0.01\%     & +0.28\%     \\
                                    & BDG2    & 9.50     & 9.35     & 9.25       & 9.39     & 9.21     & 9.16         & \textbf{8.72}         & +8.21\%  & +6.73\%   & +5.72\%    & +7.13\%     & +5.32\%     \\ \hline\hline
\end{tabular}
}\label{prob_result}
\end{table*}

\begin{figure*}[t]
\centering
\includegraphics[width=0.8\textwidth]{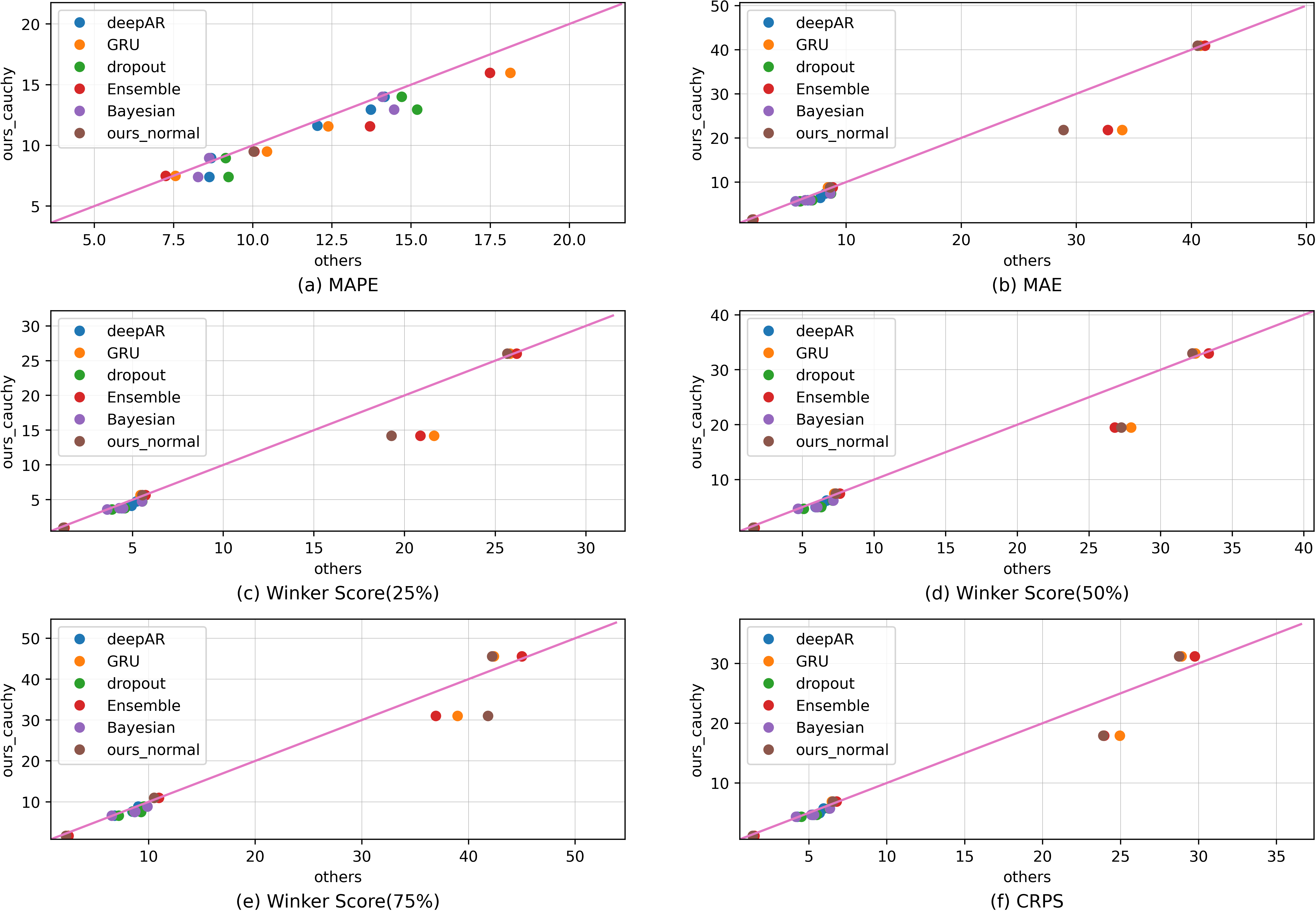}
\centering
\caption{Comparison of metrics on ten building datasets.}
\label{building}
\end{figure*}

\begin{figure}[t]
\centering
\includegraphics[width=0.5\textwidth]{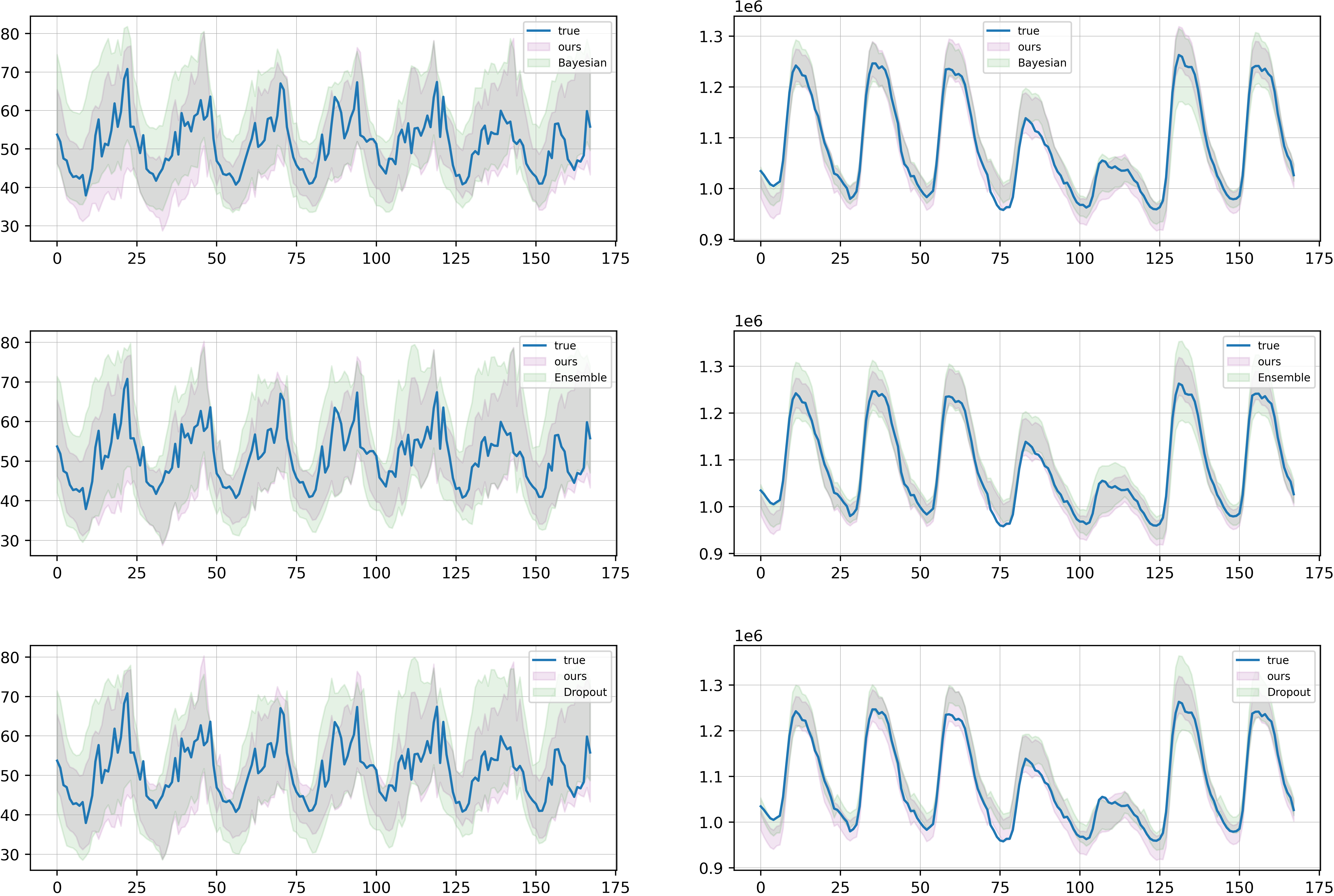}
\centering
\caption{Visualization of 75 $\%$ interval of two datasets.}
\label{prob}
\end{figure}

\subsubsection{Epistemic Uncertainty Estimation}
As shown in Fig. \ref{uncertainty}, we visualized the results of different models' estimation of the relationship between the episodic uncertainty and the number of training data on the COV dataset. The normal distribution is estimated by calculating the standard deviation, and the Cauchy distribution is estimated by calculating the quantile distance. Note that although the estimation results of different models cannot be directly compared, we can conclude the trends of each model. As the number of training data continues to increase, the model uncertainty of the dropout method has maintained fluctuating up and down, indicating that the dropout model does not learn enough about the data, resulting in the model being unable to cope with data mutations. For our model, as the amount of training data increases, the cognitive uncertainty of our model first rapidly decreases and then maintains a gentle downward trend. This indicates that our model learned a large amount of data features in the early stages, resulting in a rapid decrease in model uncertainty. Due to sudden changes in the data, there is a deviation between the training set data and the test set data. Therefore, even if the training data continues to increase, the model uncertainty does not show a significant downward trend. Compared with other methods that exhibit abnormal increases in uncertainty, our model's uncertainty estimation keeps decreasing as the data volume increases, which is more reasonable and likely to approach the true uncertainty~\cite{postels2019sampling}.

\begin{figure}[t]
\centering
\includegraphics[width=0.5\textwidth]{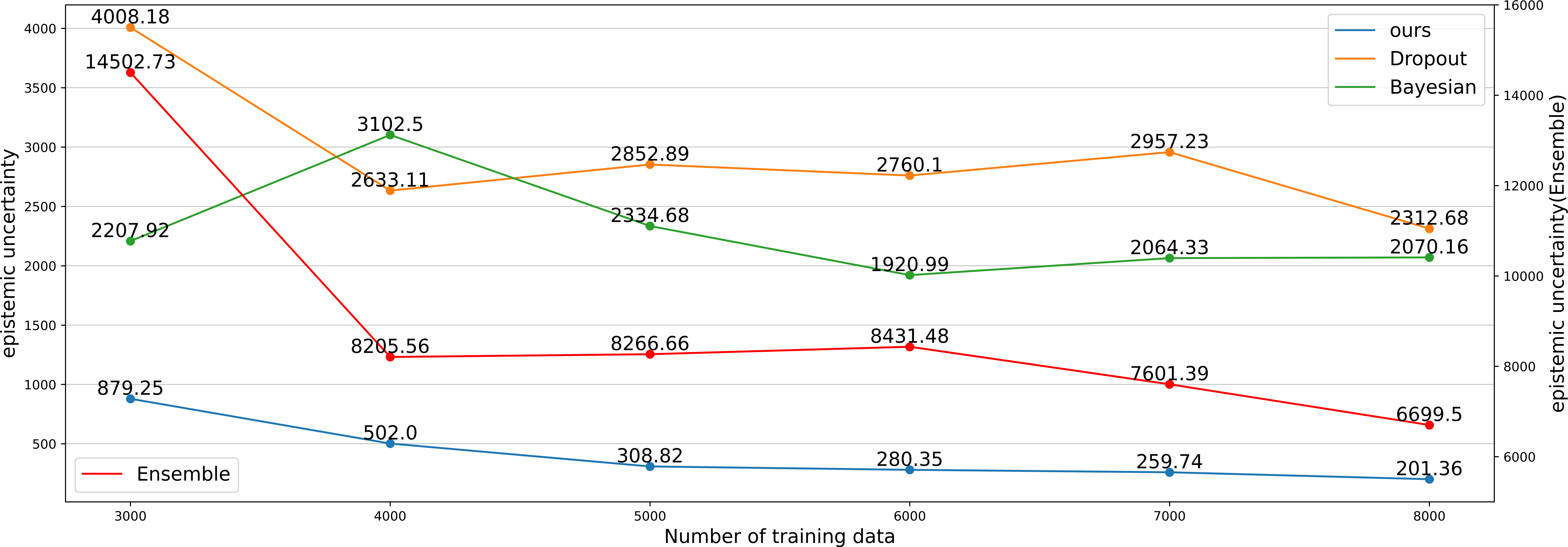}
\centering
\caption{Epistemic uncertainty estimation on COV dataset between different methods}
\label{uncertainty}
\end{figure}

\subsubsection{Robustness to outliers, noise, and data shift}
According to the definition of three types of noisy anomalies in Section \ref{aleatoric}, we perturb the training labels based on the basic training settings. Specifically, we set up three types of perturbations: 1) constant additive perturbations, which add 0.2 times the average value of all training labels on the real labels; 2) missing perturbation, replacing the real labels with the average of all training labels; 3) Gaussian perturbation, adding Gaussian noise with an expectation of 0 and a variance of 0.5 times the average of all training labels on the real labels. In implementation, we adopted two different noise ratios of 0.1 and 0.2. 

\begin{table*}[htb]\Huge
\centering
\caption{Forecasting robustness comparison among different noise}
\renewcommand{\arraystretch}{1.3}
\setlength{\tabcolsep}{6pt}
\resizebox{\textwidth}{!}{
\begin{tabular}{ccccccccc}
\hline\hline
Dataset                                    & Metric                               & Method     & Gaussian($\eta = 0.1$)                                & Gaussian($\eta = 0.2$)                                  & Missing($\eta = 0.1$)                                   & Missing($\eta = 0.2$)                                   & Constant additive($\eta = 0.1$)                         & Constant additive($\eta = 0.2$)                        \\ \hline
\multicolumn{1}{c|}{}                      &                                      & MC dropout &  3.503           &  3.513            &  3.673            &  4.361            &  3.788            &  4.679           \\
\multicolumn{1}{c|}{}                      &                                      & Bayesian   &  3.490           &  3.486            &  3.779            &  4.313            &  3.827            &  4.685           \\
\multicolumn{1}{c|}{}                      &                                      & Ensemble   &  3.436           &  3.437            &  3.560            &  4.159            &  3.626            &  4.567           \\
\multicolumn{1}{c|}{}                      &                                      & Student-T &  3.402           &  3.438            &  3.405            &  3.541            &  3.436            &  3.495           \\
\multicolumn{1}{c|}{}                      & \multirow{-5}{*}{MAPE}               & Ours       &  \textbf{3.314}  &  \textbf{3.331}   &  \textbf{3.361}   &  \textbf{3.415}   &  \textbf{3.386}   &  \textbf{3.418}  \\ \cline{2-9} 
\multicolumn{1}{c|}{}                      &                                      & MC dropout &  125.834         &  126.037          &  144.935          &  163.353          &  159.674          &  186.543         \\
\multicolumn{1}{c|}{}                      &                                      & Bayesian   &  125.389         &  125.415          &  146.816          &  166.455          &  156.882          &  188.738         \\
\multicolumn{1}{c|}{}                      &                                      & Ensemble   &  124.639         &  124.790          &  146.076          &  165.525          &  167.456          &  201.579         \\
\multicolumn{1}{c|}{}                      &                                      & Student-T &  125.419           &  126.995            &  123.320            &  137.679            &  \textbf{125.449}            &  \textbf{138.606}           \\
\multicolumn{1}{c|}{\multirow{-8}{*}{GEF}} & \multirow{-5}{*}{Winker Score(75\%)} & Ours       &  \textbf{121.090} &  \textbf{121.350} &  \textbf{122.827} &  \textbf{132.098} &  127.143 &  140.790 \\ \hline
\multicolumn{1}{c|}{}                      &                                      & MC dropout &  2.449           &  2.449            &  2.733            &  3.183            &  3.167            &  4.346           \\
\multicolumn{1}{c|}{}                      &                                      & Bayesian   &  1.962           &  1.962            &  2.211            &  3.047            &  3.718            &  5.151           \\
\multicolumn{1}{c|}{}                      &                                      & Ensemble   &  1.988           &  1.988            &  2.235            &  2.678            &  2.500            &  3.503           \\
\multicolumn{1}{c|}{}                      &                                      & Student-T &  2.272           &  2.272            &  2.253            &  2.252            &  2.312            &  2.240           \\
\multicolumn{1}{c|}{}                      & \multirow{-5}{*}{MAPE}               & Ours       &  \textbf{1.912}  &  \textbf{1.912}   &  \textbf{2.022}   &  \textbf{1.963}   &  \textbf{1.987}   &  \textbf{1.985}  \\ \cline{2-9} 
\multicolumn{1}{c|}{}                      &                                      & MC dropout &  31345.18         &  31342.77          &  36157.53          &  44386.80          &  41490.24          &  60022.86         \\
\multicolumn{1}{c|}{}                      &                                      & Bayesian   &  24034.02        &  24039.80          &  32077.69          &  40005.05          &  47984.88          &  59775.89         \\
\multicolumn{1}{c|}{}                      &                                      & Ensemble   &  23885.85         &  23885.84          &  26183.87          &  32233.43          &  28498.04          &  39111.33         \\
\multicolumn{1}{c|}{}                      &                                      & Student-T &  30143.72           &  30153.21            &  29322.96            &  29061.44            &  30336.88            &  29379.68           \\
\multicolumn{1}{c|}{\multirow{-8}{*}{COV}} & \multirow{-5}{*}{Winker Score(75\%)} & Ours       &  \textbf{23401.74} &  \textbf{23411.95} &  \textbf{24278.46} &  \textbf{24312.98} &  \textbf{24152.79} &  \textbf{24969.68} \\ \hline
\multicolumn{1}{c|}{}                      &                                      & MC dropout &  12.468           &  12.378            &  13.898           &  15.391            &  13.293            &  14.299           \\
\multicolumn{1}{c|}{}                      &                                      & Bayesian   &  12.275           &  12.277            &  13.785            &  16.295            &  12.888            &  14.248           \\
\multicolumn{1}{c|}{}                      &                                      & Ensemble   &  12.103          &  12.119            &  13.321            &  15.306            &  12.670            &  13.554           \\
\multicolumn{1}{c|}{}                      &                                      & Student-T &  11.146           &  11.231            &  11.558            &  12.360            &  11.483            &  12.447           \\
\multicolumn{1}{c|}{}                      & \multirow{-5}{*}{MAPE}               & Ours       &  \textbf{10.795}  &  \textbf{10.891}   &  \textbf{10.923}   &  \textbf{11.145}   &  \textbf{10.952}   &  \textbf{11.334}  \\ \cline{2-9} 
\multicolumn{1}{c|}{}                      &                                      & MC dropout &  14.331         &  14.582          &  15.063         &  15.674          &  14.738          &  14.959        \\
\multicolumn{1}{c|}{}                      &                                      & Bayesian   &  14.663         &  14.686          &  15.310          &  16.030          &  14.881          &  15.395         \\
\multicolumn{1}{c|}{}                      &                                      & Ensemble   &  14.405         &  14.440          &  15.157          &  15.770          &  14.823          &  15.344         \\
\multicolumn{1}{c|}{}                      &                                      & Student-T &  13.805           &  13.783           &  14.294            &  14.509            &  14.286            &  14.491           \\
\multicolumn{1}{c|}{\multirow{-8}{*}{BDG2}} & \multirow{-5}{*}{Winker Score(75\%)} & Ours       &  \textbf{13.561} &  \textbf{13.671} &  \textbf{14.063} &  \textbf{14.469} &  \textbf{13.849} &  \textbf{14.363} \\ \hline\hline
\end{tabular}
}\label{noise_result}
\end{table*}

Table \ref{noise_result} reports different forecasting results which are trained on the perturbed training set and then tested on the test set. In addition to the uncertainty quantification methods mentioned above, we replace the emission head of deepAR with student-T distribution to examine the robustness of heavy-tailed distribution to data noise and outliers. Note that the Student-T distribution is not stable. Therefore, we can not directly combine it with methods such as MC dropout to quantify epistemic uncertainty. 
Overall, the higher the noise ratio, the worse the test results of the model will be. From the perspective of different types of noise, Gaussian noise has the least impact on the test results. This is because the expectation of Gaussian noise is 0. When there are enough samples, the neural network has a strong resistance to this type of noise. However, the other two types of noise have a significant impact on the training of the model. Replacing the training label with the average of the entire label to simulate missing values has a generally less negative impact on model training than adding a perturbation to the label, which is to simulate constant noise. For different probabilistic methods, our approach, as well as the heavy tail student-T distribution, is far less affected by the training noise compared to other probabilistic forecasting methods. This superiority is more pronounced in constant noise and missing value scenarios, especially when the ratio of noise is large ($\eta = 0.2$), our method can still maintain almost consistent results with clean labels, with a performance loss of about $\mathbf{4\%}$ in deterministic forecasting while $\mathbf{16\%}$ in probabilistic forecasting. For other methods, the best-performing method with Gaussian distribution has a performance loss ratio of up to $\mathbf{33\%}$ and $\mathbf{48\%}$ in different types of forecasting, respectively. This phenomenon precisely explains why heavy-tail distribution methods perform better than other probabilistic forecasting methods in building-level data, which have relatively poor data quality and high noise. Furthermore, it is worth noting that our method has consistent performance advantages in both deterministic and probabilistic forecasting in different types of noise.
On the contrary, other methods, such as the ensemble method, although suboptimal in deterministic forecasting, perform significantly worse than other competing methods in probabilistic forecasting scenarios with constant noise. As for the comparison between our method and the Student-T distribution, our method slightly outperforms the Student-T distribution on the GEF and Building datasets. On the COV dataset, the advantages of our method become very obvious, mainly due to the significant shifts between the training and test data in the COV dataset (not just noise and outliers). This corresponds to situations with high epistemic uncertainty. With the stability of Cauchy distribution and our proposed diffusion structure, it can reduce the harm of data shifts to the model in such situations. However, for Student-T, due to its lack of stability, it is difficult to directly combine it with existing methods for quantifying epistemic uncertainty, which leads to its poor performance on COV datasets.

\subsubsection{Ablation Study}
To clarify the effectiveness of the model components, we will list additional ablation results in Table \ref{ablation}. Our vanilla model here is actually the DeepAR model. From the comparison results, it can be seen that our method leads the Vanilla model in all data and metrics. As shown in Fig. \ref{distribution_compare}, even though the confidence intervals of both distributions can roughly contain the true values, forecasts based on Gaussian distributions may exhibit abnormally large confidence intervals in certain areas, resulting in lower forecasting accuracy.

\begin{figure}[htb]
\centering
\includegraphics[width=0.5\textwidth]{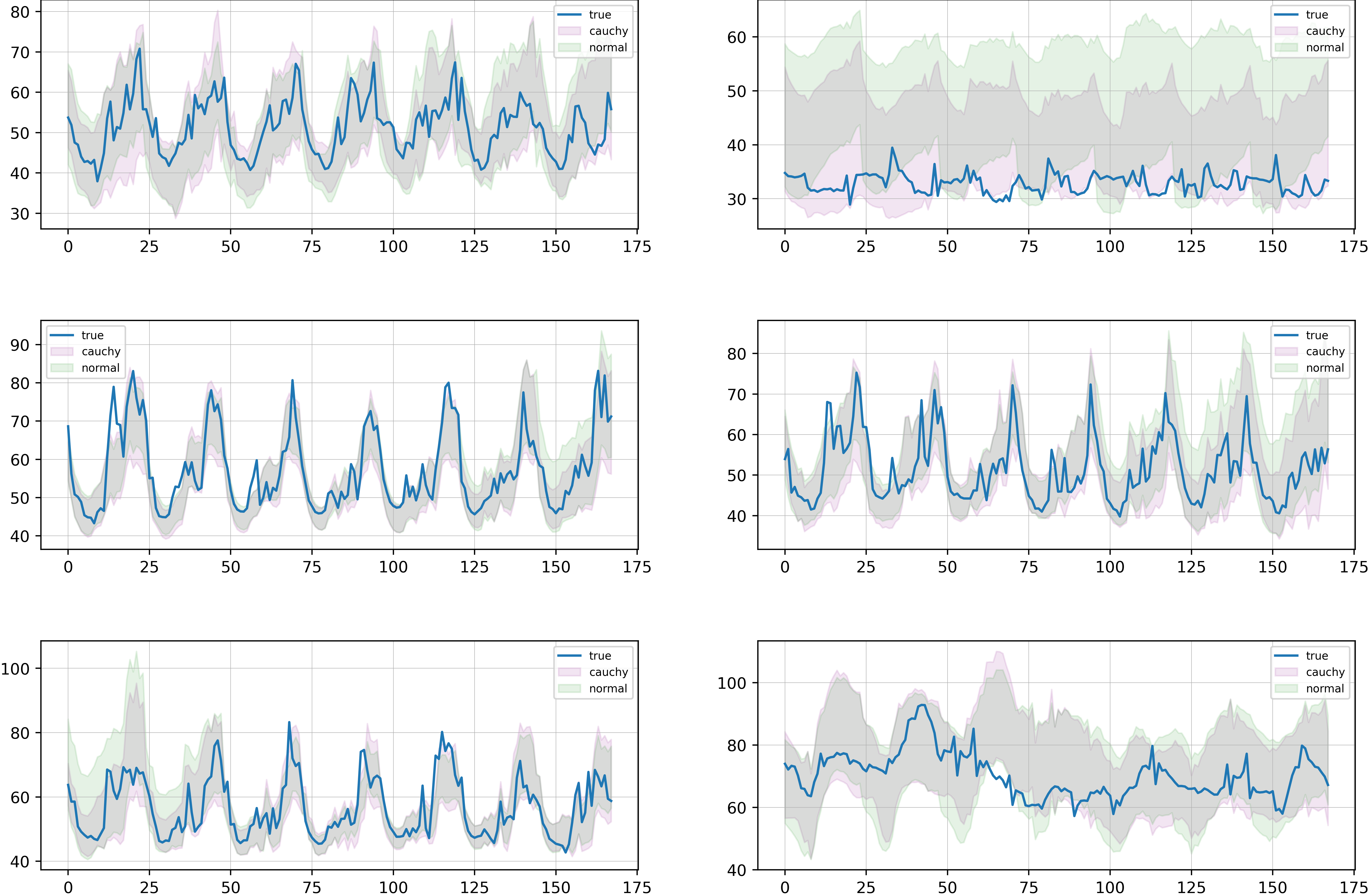}
\centering
\caption{Visualization of comparison between Gaussian and Cauchy. }
\label{distribution_compare}
\end{figure}

\begin{table}[htb]
\centering
\caption{Ablation study}
\begin{threeparttable}
\begin{tabular}{c|c|ccc}
\hline\hline
Dataset & Metric                             & o/o      & d/o            & d/c               \\ \hline
GEF     & \multirow{3}{*}{MAPE}              & 3.58     & 3.44           & \textbf{3.29}     \\
COV     &                                    & 2.04     & 2.13           & \textbf{1.91}     \\
BDG2    &                                    & 12.46    & 11.95          & \textbf{10.83}    \\ \hline
GEF     & \multirow{3}{*}{MAE}               & 119.47   & 115.37         & \textbf{110.44}   \\
COV     &                                    & 21844.0  & 22512.37       & \textbf{20308.3}  \\
BDG2    &                                    & 12.76    & 12.38          & \textbf{11.18}    \\ \hline
GEF     & \multirow{3}{*}{Winker Score(25\%)} & 75.88    & 73.39          & \textbf{70.69}    \\
COV     &                                    & 13966.92 & 14398.77       & \textbf{13063.04} \\
BDG2    &                                    & 8.14     & 7.93           & \textbf{7.17}     \\ \hline
GEF     & \multirow{3}{*}{Winker Score(50\%)} & 96.66    & 93.22          & \textbf{90.94}    \\
COV     &                                    & 18069.34 & 18601.00       & \textbf{17112.44} \\
BDG2    &                                    & 10.48    & 10.28          & \textbf{9.36}     \\ \hline
GEF     & \multirow{3}{*}{Winker Score(75\%)} & 131.00   & 124.18         & \textbf{121.03}   \\
COV     &                                    & 25408.52 & 25988.62       & \textbf{23399.30} \\
BDG2    &                                    & 14.39    & 14.26          & \textbf{13.60}    \\ \hline
GEF     & \multirow{3}{*}{CRPS}              & 86.61    & \textbf{83.07} & 83.24             \\
COV     &                                    & 16263.42 & 16705.65       & \textbf{15435.79} \\
BDG2    &                                    & 9.35     & 9.16           & \textbf{8.72}     \\ \hline\hline
\end{tabular}
\begin{tablenotes}
\footnotesize
\item[1] o/o: without diffusion structure and use the Gaussian distribution. 
\item[2] d/o: with diffusion structure and use the Gaussian distribution. 
\item[3] d/c: with diffusion structure and use the Cauchy distribution.
\end{tablenotes}
\end{threeparttable}
\label{ablation}
\end{table}

\subsubsection{Time Consuming Comparison}\label{time_eval}
Although uncertainty estimation based on methods such as ensemble and Bayesian neural networks can improve forecasting accuracy (compared to simple DeepAR), the additional computational costs cannot be ignored. Assume that the parameter quantity of the GRU module is $p$, Table \ref{Time} shows the complexity and the time required for one training on ten building datasets, where $\gamma_1$ and $\gamma_2$ represent the number of additional parameters of the network that are used to capture posterior distribution and $K$ represents the number of the models that are trained in ensemble strategy. Since the Bayesian method requires modeling the distribution of each parameter, the extra complexity will be related to the number of the original model parameters while our methods just need a small amount of extra parameters. Because the deep ensemble approach necessitates training numerous separate neural networks, its training time increases linearly when compared to standard DeepAR. Similarly, the Bayesian neural network technique treats every parameter of the NN as a random variable with a minimum order of magnitude of a million. Therefore, it is no wonder that these methods require much more computational resources than other methods. The dropout approach is less computationally intensive than other uncertainty estimating techniques, but it might result in inconsistent training and inference processes, which can quickly reduce forecasting accuracy~\cite{li2019understanding} (also shown in the poor performance in the COV dataset). When compared to the Bayesian and Ensemble approaches, our method does not significantly differ from the simple model, even though it takes more computing time than the simple deepAR model because of the addition of diffusion structures.

\begin{table}[htb]\Huge
\centering
\caption{Time consumed by different methods.}
\resizebox{0.5\textwidth}{!}{
\begin{tabular}{cccccccc}
\hline
\hline
Method  & GRU & DeepAR & Dropout & Ensemble & Bayesian   & Ours Gaussian & Ours Cauchy \\ \hline
Time(s) & 75.91 & 107.83 & 106.62  & 8880.56  & 20064.76 & 242.01        & 197.71       \\ \hline
Complexity & $O(p)$ & $O(p)$   & $O(p)$    & $O(Kp)$    & $O(p\gamma_1)$ & $O(p+\gamma_2)$     & $O(p+\gamma_2)$ \\ \hline\hline
\end{tabular}
}\label{Time}
\end{table}

\section{Conclusions}\label{Conclusion}
This paper proposes a novel method for estimating uncertainty and applying it to load forecasting. On the one hand, we employ a robust heavy-tailed Cauchy distribution to encapsulate our forecasting model, reducing the model's sensitivity to outliers and sudden changes. This approach ensures the model's training stability while estimating aleatoric uncertainty. On the other hand, we propose a novel empirical estimation method based on the diffusion model. Unlike traditional methods focusing on the model's parameters, our approach utilizes the hidden state in Seq2Seq to estimate uncertainty, significantly reducing the training time and providing superior probabilistic forecasts. In future work, we will consider combining this estimation method with online learning to achieve higher forecasting accuracy through adaptive model updates based on uncertainty estimation.

\bibliographystyle{IEEEtran}


\vspace{-0.5cm}
\begin{IEEEbiography}
[{\includegraphics[width=1in,height=1.25in,clip,keepaspectratio]{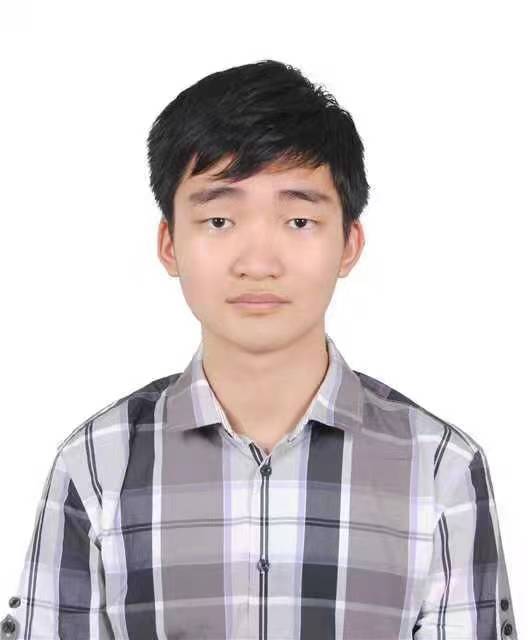}}]{Zhixian Wang} received the B.S. degree in Statistics from the University of Science and Technology of China, Hefei, China, in 2022. He is now pursuing a Ph.D. degree in Electrical and Electronic Engineering at the University of Hong Kong. His current research interests include energy forecasting and explainable AI in smart grids.
\end{IEEEbiography}
\begin{IEEEbiography}
[{\includegraphics[width=1in,height=1.25in,clip,keepaspectratio]{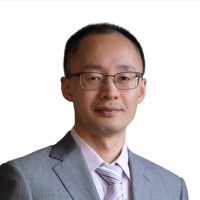}}]{Qingsong Wen} is the Head of AI Research \& Chief Scientist at Squirrel Ai Learning. He holds a Ph.D. in Electrical and Computer Engineering from Georgia Institute of Technology, USA. He has published over 100 top-ranked AI conference and journal papers, had multiple Oral/Spotlight Papers at NeurIPS, ICML, and ICLR, had multiple Most Influential Papers at IJCAI, received multiple IAAI Deployed Application Awards at AAAI, and won First Place in SP Grand Challenge at ICASSP.  Currently, he serves as Co-Chair of the Workshop on AI for Time Series, Workshop on AI for Education, and Workshop on AI Agent for Information Retrieval at KDD, AAAI, IJCAI, etc. He also serves as Associate Editor for Neurocomputing and IEEE Signal Processing Letters, and regularly serves as Area Chair of the top conferences including NeurIPS, KDD, ICASSP, etc. His research focuses on AI for time series, AI for education, and general machine learning.
\end{IEEEbiography}
\begin{IEEEbiography}
[{\includegraphics[width=1in,height=1.25in,clip,keepaspectratio]{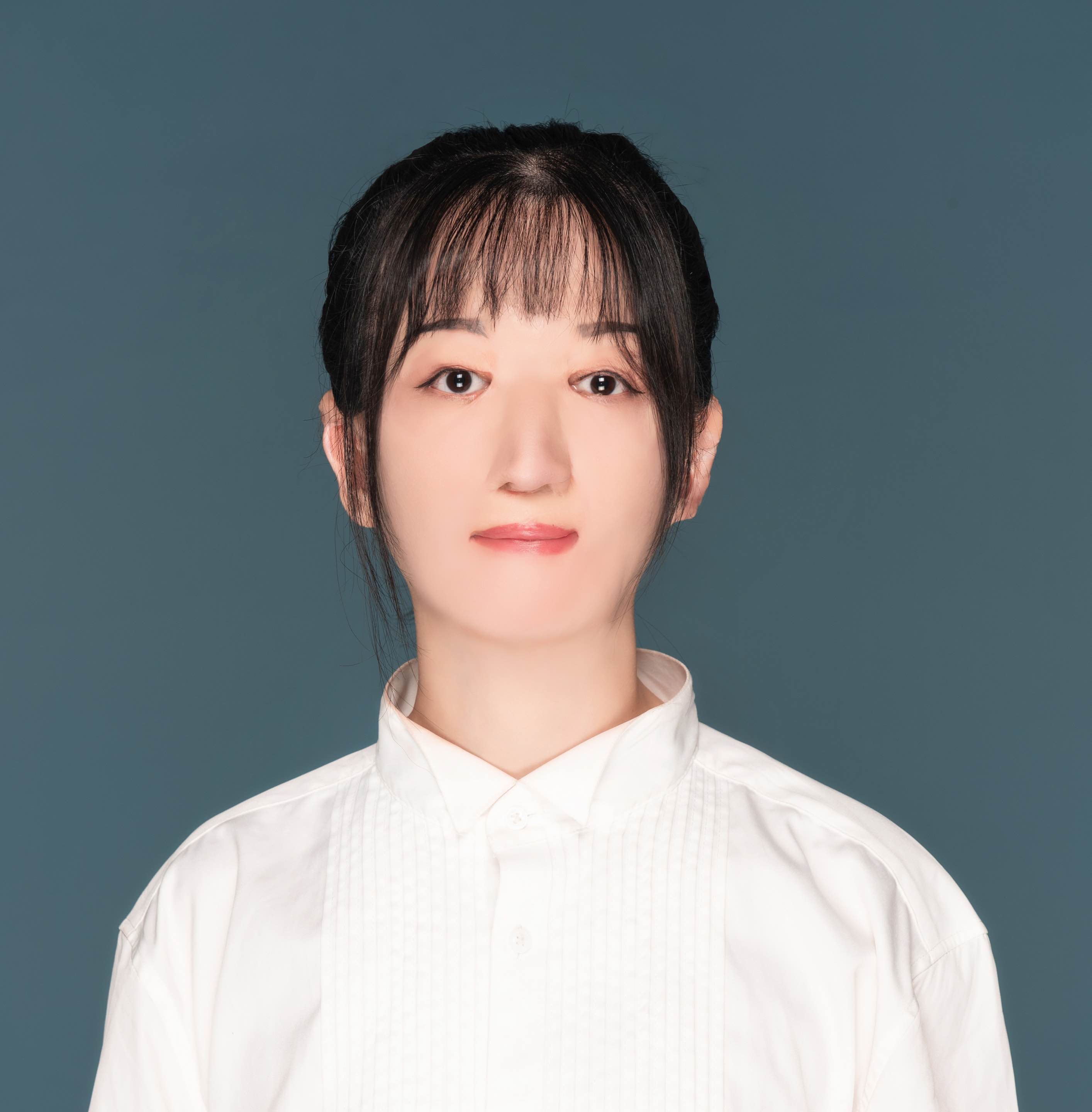}}]{Chaoli Zhang} received the B.S. degree in information security from Nankai University, China, in 2015, and the Ph.D. degree in computer science and engineering from Shanghai Jiao Tong University, China in 2020. She is currently a lecturer at the School of Computer Science and Technology, at Zhejiang Normal University. She engaged in research at Alibaba DAMO Academy for nearly three years. Her research interests include time series analysis, algorithmic game theory and mechanism design, and networking. She won the gold prize in ICASSP-SPGC root cause analysis for wireless network fault localization in 2022. She was the recipient of the Google Anita Borg Scholarship 2014 and the AAAI/IAAI Innovative Deployed Application Award 2023. 
\end{IEEEbiography}
\begin{IEEEbiography}[{\includegraphics[width=1in,height=1.25in,clip,keepaspectratio]{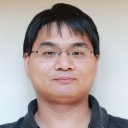}}]{Liang Sun} received the B.S. degree in computer science from Nanjing University and the Ph.D. degree from in computer science, Arizona State University. He is currently a Senior Staff Engineer/Engineering Director with DAMO Academy-Decision Intelligence Laboratory, Alibaba Group, where he is working on building a data-driven decision-making cycle in automated business analysis, including data monitoring, insights discovery, diagnosis and root cause analysis, action suggestion, and explainability of the cycle, with an emphasis on time series data. He has over 50 publications, including two books in the fields of machine learning and data mining. His work on dimensionality reduction won the KDD 2010 Best Research Paper Award Honorable Mention and won Second Place in the KDD Cup 2012 Track 2 Competition. He also won First Place in the 2022 ICASSP Grand Challenge (AIOps in Networks) Competition and received the AAAI/IAAI 2023 Deployed Application Award.
\end{IEEEbiography}
\begin{IEEEbiography}
[{\includegraphics[width=1in,height=1.25in,clip,keepaspectratio]{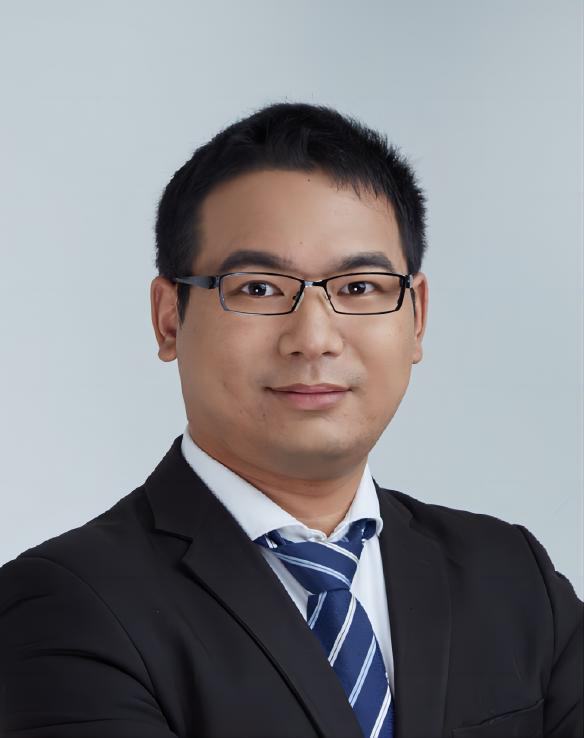}}]{Yi Wang} received the B.S. degree from Huazhong University of Science and Technology in June 2014, and the Ph.D. degree from Tsinghua University in January 2019. He was a visiting student at the University of Washington from March 2017 to April 2018. He served as a Postdoctoral Researcher in the Power Systems Laboratory, at ETH Zurich from February 2019 to August 2021. 

He is currently an Assistant Professor at the Department of Electrical and Electronic Engineering at the University of Hong Kong. His research interests include data analytics in smart grids, energy forecasting, multi-energy systems, Internet-of-things, and cyber-physical-social energy systems.
\end{IEEEbiography}
\vfill
\end{document}